\documentclass[lettersize,journal]{IEEEtran}
\usepackage{cite}
\usepackage{amsmath,amssymb,amsfonts}
\usepackage{graphicx}
\usepackage{textcomp}
\usepackage{tabu}
\usepackage{booktabs}
\usepackage{multirow}
\usepackage{amsmath}
\usepackage{epstopdf}
\usepackage{pbox}
\usepackage{algorithm}
\usepackage{CJK}
\usepackage{mathtools}%

\newtheorem{theorem}{\textbf{Theorem}}

\newcounter{step}
\renewcommand{\thestep}{S\arabic{step}}

\newcommand{\skipthis}[1]{ }

\usepackage{color}
\usepackage[short]{optidef}
\usepackage[compatible]{algpseudocode}
\usepackage{flushend}

\newtheorem{proposition}{\textbf{Proposition}}

\newtheorem{remark}{\textbf{Remark}}
\newenvironment{proof}{\noindent{\em Proof:\/}}{\hfill $\blacksquare$\par}

\newcommand{\argmin}{\operatorname*{arg\,min}}

\begin{document}

\title{Real-Time Auto-Optimization in Unknown Environments via \\ Structure-Exploiting Dual Control for Exploration and Exploitation}

\author{Shiying Dong,
Haoyang Yang,
Qiwei Liu,
Wen-Hua Chen, \emph{IEEE Fellow}
\thanks{This work is supported by the Research Centre for Low Altitude Economy (RCLAE), the Hong Kong Polytechnic University, and in part by the Smart Traffic Fund of the Transport Department of the Hong Kong Special Administrative Region under Project PSRI/131/2512/PR. (Corresponding author: Wen-Hua Chen)}
\thanks{S. Dong, H. Yang, Q. Liu, and W. H. Chen are with the Department of Aeronautical and Aviation Engineering and Research Centre for Low Altitude Economy (RCLAE), the Hong Kong Polytechnic University, Hong Kong, China. (\{shiying.dong, haoyang.yang, qiwei-rclae.liu, wenhua.chen\}@polyu.edu.hk)}
}

\maketitle

\begin{abstract}
This paper develops a fast numerical dual control for exploration and exploitation (DCEE) method to address auto-optimization problems in unknown environments. 
In auto-optimization problems, the optimal operating condition is unknown a priori and may vary with the environment. 
As in classical dual control techniques, computational burden remains a major concern in DCEE for active learning.
Existing DCEE methods provide a principled exploration-exploitation objective, but mainly realized through standard optimization packages or explicit gradient-type update laws, where the numerical structure of the DCEE has not been fully exploited. 
This paper shows that the reward function in DCEE has an inherent convex-over-nonlinear structure, where the exploitation and exploration terms form a unified nonlinear residual map equipped with a convex outer loss. 
Benefiting from this structure, a structure-exploiting numerical method is developed by linearizing only the nonlinear residual map while preserving the convex outer loss. 
Thus, each subproblem is transformed into a structured convex form that can be solved reliably.
The resulting generalized Gauss-Newton Hessian approximation is positive semidefinite and depends only on first-order derivatives, thereby supporting fast online computation.
The proposed method is evaluated on a vehicle cruising auto-optimization problem and compared with existing methods.
Simulation and hardware-in-the-loop experimental results show that the proposed method improves control performance and achieves a speedup of approximately one order of magnitude, with a microsecond-level maximum computation time of only \(83~\mu\mathrm{s}\) on a typical vehicle embedded CPU.
\end{abstract}



\section{Introduction}

Autonomous systems are increasingly required to operate in environments where the optimal operating conditions are unknown, uncertain, and time-varying \cite{annaswamy2024control}. 
In autonomous vehicle eco-cruising, unknown road-load variations can shift the speed-force operating point and change the energy-efficient cruising speed~\cite{sciarretta2020energy,han2019fundamentals,zhang2023energy}.
This situation also appears in applications such as autonomous source seeking \cite{hutchinson2018information,rhodes2022autonomous}, motor efficiency optimization \cite{zhu2025adaptive},  energy system optimization \cite{tang2026auto}, and motion control \cite{sun2025adaptive}, where the system must continuously adjust its behavior as the surrounding environment changes. 
Therefore, auto-optimization control aims to drive the system to an unknown optimal operating state while updating this state through interaction with the environment.

The key difficulty is that the controller is not given a reference to track.
Set-point tracking model predictive control (MPC)~\cite{rawlings2017model} has been widely studied for constrained optimal control and has been applied in industrial process control, autonomous driving, and robotic navigation. 
Classical tracking MPC usually assumes that the reference, model parameters, or performance map are available before operation, and then computes control inputs to track the prescribed set-point while satisfying system constraints \cite{reiter2026synthesis}.
However, in many auto-optimization problems, the desired set-point is not known in advance.
It depends on unknown or time-varying environmental parameters and must be estimated from online measurements during operation.
Adaptive, learning-based, and active-learning MPC \cite{saviolo2023active} address this problem by using operational data to update uncertain models or parameters, while self-optimizing MPC further introduces future data and excitation mechanisms to balance optimality tracking and parameter identification~\cite{tan2025auto}. 
Reinforcement learning provides another route through repeated interaction or offline training, but for real-time systems with limited data and a need for interpretability, control-theoretic online optimization remains attractive~\cite{chen2022perspective}.
 
For auto-optimization, dual control for exploration and exploitation (DCEE) provides a way to move the system toward an unknown optimal operating condition while actively learning the reward environment online~\cite{li2025dual,yu2026k}.
DCEE was originally developed for autonomous search problems in unknown environments, where the control input affects both the motion of the search agent and the information collection used for source estimation~\cite{chen2021dual}. 
Subsequent studies reduced the computational burden of particle-filter-based implementations by introducing concurrent active learning with multiple estimators~\cite{li2022concurrent}. 
Multi-step DCEE further incorporates prediction over a finite horizon and establishes recursive feasibility and convergence properties~\cite{tan2024multistep}, while Nesterov-accelerated DCEE improves computational efficiency~\cite{tan2025dual}. 
Cooperative active-learning DCEE extends this idea to multi-agent search problems~\cite{li2024cooperative}, and recent applications show that the relevance of DCEE has gone beyond the search domain, including robotic assembly~\cite{pashupathy2026dual} and autonomous sensorless control under operational uncertainty~\cite{zhang2026autonomous}.

Despite these remarkable advances, the computational burden of dual control remains a fundamental challenge. 
Classical dual control explicitly accounts for the fact that the control input has both a directing effect on the system state and a probing effect on future uncertainty, which typically leads to nonlinear and belief-state-dependent optimization problems \cite{wittenmark2008adaptive}. 
Dual control has long been regarded as computationally intractable even for moderately sized systems, and extensive research has been devoted to suboptimal and approximate solutions \cite{filatov2000survey,wittenmark2008adaptive}. 
The heavy computational burden has also been explicitly recognized in recent DCEE studies, where reducing online computational load is identified as an important research direction~\cite{chen2022perspective}.
To address this issue, recently explicit gradient-type control laws, rather than directly solving the online nonlinear optimization problem involved in DCEE, have been proposed in~\cite{li2025dual,yu2026k}. 
They have low computational cost and are convenient for analysis, but they do not fully exploit the numerical structure of the optimization problem involved in DCEE and the performance may be significantly degraded due to the loss of optimality. 
In this approach, the reward function in DCEE is usually differentiated to obtain an update direction, rather than being directly solved in the physical input space as a structured nonlinear optimization problem. 
Therefore, the nonlinear dependence of the predicted output, the predicted active-learning update, and the predicted optimal operating condition on the candidate input is only partially reflected in the implemented control law. 

Modern numerical optimal control has shown that exploiting problem structure is crucial for real-time implementation. 
In particular, sequential convex programming (SCP), Gauss-Newton (GN) methods, generalized Gauss-Newton (GGN) Hessian approximations, and structure-preserving MPC formulations can transform complex nonlinear programs into efficient local convex or least-squares subproblems~\cite{messerer2021survey,diehl2019local}.
Inspired by the recent advances in this field, this paper aims to develop a structure-exploiting numerical reformulation of DCEE that could significantly reduce the computational burden while preserving the optimality of the original DCEE.
The key observation of this paper is that the DCEE has an inherent convex-over-nonlinear structure. 
By stacking the exploitation residual and the ensemble uncertainty residual, the DCEE can be written as a composition of a convex quadratic outer loss function and a nonlinear residual mapping.  
Based on this structure, we propose a direct numerical DCEE method which can provide a fast and accurate solution. 

The main contributions of this paper are summarized as follows:
\begin{itemize}

    \item We identify an inherent convex-over-nonlinear structure of DCEE, where the exploitation and exploration terms are stacked into a unified nonlinear residual map with a convex quadratic outer loss.
    
    \item Benefiting from the numerical structure of the DCEE, this paper introduces a structure-exploiting SCP method for numerically solving the DCEE problem. 
    By linearizing only the nonlinear DCEE residual mapping while preserving the convex outer loss, we show that the sequential convex approximation of the DCEE problem reduces exactly to a Gauss-Newton least-squares step. 
    Therefore, each inner iteration solves a structured convex subproblem rather than an arbitrary nonconvex approximation.
    
    \item We further analyze the proposed structure-exploiting SCP-DCEE method from the generalized Gauss-Newton viewpoint. 
    The Hessian approximation of each subproblem is positive semidefinite, requires only first-order derivatives of the DCEE residual mapping, and avoids second-order derivatives, thereby supporting fast and accurate numerical solution.

\end{itemize}
The proposed method is validated on a vehicle cruising auto-optimization problem. 
The closed-loop behavior is compared with classical DCEE~\cite{li2025dual} and extremum seeking control, and the results show that the proposed method achieves better transient response, lower cumulative tracking error, and smaller cumulative regret. 
The computational performance is further compared with a direct nonlinear programming implementation using \texttt{IPOPT} \cite{wachter2006implementation} and \texttt{CasADi} \cite{andersson2018casadi}, demonstrating the advantage of exploiting the numerical structure. 
In addition, its real-time feasibility is evaluated through hardware-in-the-loop (HiL) experiments.
The results show that the proposed method achieves a microsecond-level maximum computation time of only \(83~\mu\mathrm{s}\) on a typical vehicle embedded CPU.

The remainder of this paper is organized as follows. 
Section II formulates the DCEE auto-optimization problem. 
Section III reveals the convex-over-nonlinear-residual structure. 
Section IV develops the structure-exploiting numerical methods. 
Section V presents the simulation and HiL experimental results. 
Section VI concludes the paper.


\section{Problem Statement}
\label{sec:problem_statement}

This section formulates the DCEE problem for auto-optimizing control in unknown environments. 
Following the standard DCEE framework, the auto-optimization objective is first rewritten as a dual control objective that embeds both exploitation and exploration effects. 
Then, an ensemble-learning-based active learning mechanism is introduced to approximate the conditional mean and uncertainty of the unknown optimal operating condition.


\subsection{Dual Control Reformulation}
\label{subsec:dual_control_reformulation}

Consider a system operating in an unknown environment. 
The effect of the environment on the performance is represented by an unknown parameter $\theta^\ast\in\mathbb R^{n_\theta}$.
For any candidate parameter \(\theta\), the reward associated with output
\(y\in\mathbb R^{n_y}\) is modeled as
\begin{equation}
\mathcal R(\theta,y)
=
\psi(y)^\top \theta+\psi_0(y),
\label{eq:reward_model}
\end{equation}
where \(\psi(\cdot)\) is a known basis function and \(\psi_0(\cdot)\) is a known offset term. 
The reward function induced by the true environment is therefore \(\mathcal R(\theta^\ast,y)\), but \(\theta^\ast\) is unknown and needs to be learned from online reward measurements. 
Without loss of generality, the desired operating condition is defined as the maximizer of the performance index \(\mathcal R\). 
The goal of auto-optimization is to guide the system toward this unknown optimum and to track the optimal condition when the environment changes.

The system dynamics are described by
\begin{equation}
x_{k+1}=f(x_k, u_k), \quad y_k=Cx_k,
\label{eq:plant}
\end{equation}
where \(x_k\in\mathbb R^{n_x}\), \(u_k\in\mathbb R^{n_u}\). 
At each time \(k\), the system output and the reward can be measured, subject to measurement noise:
\begin{equation}
\mathcal R_k = \mathcal R(\theta_k,y_k)+\varepsilon_k,
\label{eq:measured_reward}
\end{equation}
where \(\varepsilon_k\) denotes measurement noise.
The measurement vector is defined as $z_k=[y_k, \mathcal R_k]$.
The instantaneous information state at time step \(k\) is given by
\begin{equation}
\mathcal I_k:=\{u_{k-1},z_k\},
\label{eq:instantaneous_information_state}
\end{equation}
which consists of the previously applied control input \(u_{k-1}\) and the current measurement vector \(z_k\).
The collection of all information available up to time step \(k\) is denoted by
\begin{equation}
\mathcal H_k=\{\mathcal I_0, \mathcal I_1,\ldots, \mathcal I_k\},
\label{eq:information_history}
\end{equation}
with the initial information defined as $\mathcal I_0=z_0$.

For a candidate input \(u\), the one-step predicted state and output are
\begin{equation}
x_{k+1|k}(u)=f(x_k, u), \quad y_{k+1|k}(u)=Cx_{k+1|k}(u).
\label{eq:prediction}
\end{equation}
The candidate future information packet is denoted by
\begin{equation}
\mathcal I_{k+1|k}(u) := \{u,  y_{k+1|k}(u),\mathcal R_{k+1|k}(u)\},
\end{equation}
and the candidate information set is $\mathcal H_{k+1|k}(u) := \{\mathcal H_k,\mathcal I_{k+1|k}(u)\}$.
The explicit dependency on \(u\) emphasizes that candidate control inputs influence the predicted output and the predictive information available for learning. 
In the following paper, this dependency will be omitted where there is no ambiguity.

As introduced in \cite{li2025dual}, the auto-optimization problem admits two DCEE-based formulations.
The first formulation is close to extremum seeking control. 
Given the prior and the measurement history up to time \(k\), the control input is chosen to maximize the predicted reward
\begin{equation}
u_k^\ast \in \arg\max_{u\in\mathbb R^{n_u}} \; \mathbb E \left[ \mathcal R(\theta^\ast,y_{k+1|k}) \middle| \mathcal H_{k+1|k} \right].
\label{eq:reward_max_dcee_form}
\end{equation}

The second formulation does not optimize the reward value directly. 
It instead drives the system output toward the unknown optimal operating condition. 
Because the true optimum is unavailable, the controller tracks its best prediction based on the information available at the decision time.
Then the control input is obtained from
\begin{equation}
u_k^\ast \in \arg\min_{u\in\mathbb R^{n_u}} \; \mathbb E \left[ \left\| y_{k+1|k}-\gamma^\ast \right\|^2 \middle| \mathcal H_{k+1|k} \right].
\label{eq:tracking_dcee_form}
\end{equation}
Here, $\Gamma:\mathbb R^{n_\theta}\rightarrow\mathbb R^{n_y}, \ \gamma^\ast=\Gamma(\theta^\ast)$ gives the predicted optimal operating condition associated with the environment parameter \(\theta^\ast\). 
We assume that, for each \(\theta\) in the region of interest, the performance objective admits a unique local maximizer, denoted by \(\Gamma(\theta)\), and that \(\Gamma(\cdot)\) is smooth.

This paper is centered on the second formulation in \eqref{eq:tracking_dcee_form}. 
In this setting, the optimal operating condition is not prescribed in advance, but is induced by the unknown environment parameters and must therefore be inferred during operation.

Before proceeding to the numerical solution, we show that \eqref{eq:tracking_dcee_form} naturally contains dual effects. 
Define
\begin{equation}
D_k := \mathbb E \left[ \left\| y_{k+1|k}-\gamma^\ast \right\|^2 \middle| \mathcal H_{k+1|k} \right].
\label{eq:ideal_dcee}
\end{equation}
The predicted nominal optimal operating condition is
\begin{equation}
\bar\gamma_{k+1|k} := \mathbb E \left[ \gamma_{k+1|k} \middle| \mathcal H_{k+1|k} \right],
\label{eq:conditional_gamma_mean}
\end{equation}
and the corresponding prediction error is
\begin{equation}
\tilde\gamma_{k+1|k} := \gamma^\ast-\bar\gamma_{k+1|k}.
\label{eq:conditional_gamma_error}
\end{equation}

Substituting $\gamma^\ast = \bar\gamma_{k+1|k} + \tilde\gamma_{k+1|k}$ into \eqref{eq:ideal_dcee} gives
\begin{align}
D_k & = \mathbb E \left[ \left\| y_{k+1|k} - \bar\gamma_{k+1|k} - \tilde\gamma_{k+1|k} \right\|^2 \right] \nonumber\\
    & = \mathbb E \left[ \left\| y_{k+1|k} - \bar\gamma_{k+1|k} \right\|^2  \right] \nonumber\\
    & - 2\mathbb E \left[ \left( y_{k+1|k} - \bar\gamma_{k+1|k} \right)^\top \tilde\gamma_{k+1|k}  \right] \nonumber\\
    & + \mathbb E \left[ \left\| \tilde\gamma_{k+1|k} \right\|^2  \right].
\label{eq:dcee_expansion}
\end{align}
By definition $\mathbb E \left[ \tilde\gamma_{k+1|k}  \right]=0$, the cross term vanishes. 
Hence,
\begin{equation}
D_k = \underbrace{ \mathbb E \left[ \left\| y_{k+1|k} - \bar\gamma_{k+1|k} \right\|^2 \right] }_{\text{exploitation}}
+ \underbrace{ \mathbb E \left[ \left\| \tilde\gamma_{k+1|k} \right\|^2 \right] }_{\text{exploration}}.
\label{eq:dcee_dual_decomposition_expectation}  
\end{equation}
The first term in \eqref{eq:dcee_dual_decomposition_expectation} drives the predicted output toward the predicted nominal optimal operating condition and therefore corresponds to exploitation or optimality tracking. 
The second term measures the uncertainty of the predicted optimal operating condition and therefore corresponds to exploration. 
The candidate input \(u\) affects both terms: it changes the predicted output through the system dynamics and changes the predicted information set through the reward observation. 
Thus, the dual effect is embedded in the expected tracking objective itself.


\subsection{Ensemble-Based Active Learning}
\label{subsec:ensemble_based_active_learning}
The unknown parameters can be updated online by efficient gradient-based algorithms.
In practice, however, a single-estimator optimization scheme is often sensitive to noisy measurements and nonlinear model errors.
Following  \cite{li2025dual}, this paper adopts a multi-estimator learning scheme for parameter adaptation.
Instead of relying on one parameter estimate, the proposed method propagates a set of estimators to represent the current uncertainty of the unknown environment.

An ensemble of \(N\) estimators is used to describe the current belief of the unknown environment, \(\theta_k^i, \ i=1,\ldots,N,\) with
\begin{equation}
\Theta_k := \operatorname{col}(\theta_k^1,\ldots,\theta_k^N), \quad \bar\theta_k := \frac1N\sum_{i=1}^N\theta_k^i.
\label{eq:theta_ensemble}
\end{equation}
For a given candidate input \(u\), each estimator gives a one-step prediction of the corresponding optimal operating condition:
\begin{equation}
\hat\gamma_{k+1|k}^i = \Gamma(\hat\theta_{k+1|k}^i),
\label{eq:predicted_gamma_i}
\end{equation}
where \(\hat\theta_{k+1|k}^i\) is the predicted value of the \(i\)-th estimator induced by the candidate input.

The nominal prediction of the optimal operating condition is then obtained by averaging the ensemble predictions,
\begin{equation}
\bar\gamma_{k+1|k} = \frac1N \sum_{i=1}^{N} \hat\gamma_{k+1|k}^i,
\label{eq:ensemble_gamma_mean}
\end{equation}
and the deviation of each ensemble member from this nominal prediction is defined as
\begin{equation}
\tilde\gamma_{k+1|k}^i = \hat\gamma_{k+1|k}^i - \bar\gamma_{k+1|k}.
\label{eq:ensemble_gamma_deviation}
\end{equation}
The spread of the predicted optimal operating condition is approximated by the sample covariance
\begin{equation}
\Sigma_{\gamma,k+1|k} = \frac1N \sum_{i=1}^{N} \tilde\gamma_{k+1|k}^i \tilde\gamma_{k+1|k}^i\top .
\label{eq:ensemble_covariance}
\end{equation}

Using \eqref{eq:ensemble_gamma_mean}-\eqref{eq:ensemble_covariance} in \eqref{eq:dcee_dual_decomposition_expectation}, the reward function in DCEE can be evaluated directly from the ensemble as
\begin{equation}
D_k = \left\| y_{k+1|k} - \bar\gamma_{k+1|k} \right\|^2 + \frac1N \sum_{i=1}^{N} \left\| \tilde\gamma_{k+1|k}^i \right\|^2 .
\label{eq:implementable_strict_dcee}
\end{equation}
Equivalently,
\begin{equation}
D_k = \left\| y_{k+1|k} - \bar\gamma_{k+1|k} \right\|^2 + \operatorname{tr} \left( \Sigma_{\gamma,k+1|k} \right).
\label{eq:strict_dcee_objective}
\end{equation}

This gives the direct numerical DCEE problem studied in this paper
\begin{equation}
u_k^\star \in \argmin_{u\in\mathbb R^{n_u}} \; D_k(u).
\label{eq:direct_dcee_problem}
\end{equation}
Compared with the DCEE implementation based on an explicit gradient-type control law \cite{li2025dual}, the above formulation optimizes the reward function in DCEE directly in the physical input space.  
In this way, the dependence of both \(y_{k+1|k}(u)\) and \(\bar\gamma_{k+1|k}(u)\) on the candidate input \(u\) is retained in the numerical optimization problem.

The predicted estimator \(\hat\theta_{k+1|k}^i\) is computed as follows.
For the linear-in-parameter reward model in \eqref{eq:reward_model}, each estimator is updated from the measured reward by
\begin{equation}
\theta_{k+1}^i = \theta_k^i - \eta_i\psi(y_k) \left( \mathcal R(\theta_k^i,y_k) - \mathcal R_k \right), i=1,\ldots,N,
\label{eq:online_estimator_update}
\end{equation}
where \(\eta_i>0\) is the learning rate of the \(i\)-th estimator.

For a candidate input \(u\), the reward at the next step has not yet been measured.
It is therefore approximated using the current ensemble mean
\begin{equation}
\hat {\mathcal R}_{k+1|k} = \mathcal R(\bar\theta_k,y_{k+1|k}).
\label{eq:predicted_reward}
\end{equation}
Substitution of this predicted reward into the update rule yields
\begin{equation}
\hat\theta_{k+1|k}^i = \theta_k^i - \eta_i \psi(y_{k+1|k}) \left( \mathcal R(\theta_k^i,y_{k+1|k}) - \hat {\mathcal R}_{k+1|k} \right).
\label{eq:predicted_estimator_update}
\end{equation}

Through \eqref{eq:predicted_estimator_update}, the candidate input \(u\) affects the predicted output, the reward regression, and the next estimator values.
The predicted optimal operating conditions obtained from the ensemble therefore also depend on \(u\). 
Hence both \(\bar\gamma_{k+1|k}\) and \(\Sigma_{\gamma,k+1|k}\) in \eqref{eq:strict_dcee_objective} vary with the same decision variable.

\begin{remark}
The update rule in \eqref{eq:predicted_estimator_update} is not essential to the following proposed numerical formulation. 
It is one possible ensemble active-learning realization. 
Other estimator prediction maps can be used, provided that the candidate-input-induced maps $u\mapsto \hat\theta_{k+1|k}^i(u), \ u\mapsto \hat\gamma_{k+1|k}^i(u)$ are well-defined and differentiable on the region of interest.
\end{remark}


\section{Convex-over-Nonlinear Structure of DCEE}
\label{sec:con_structure_dcee}

This section presents the numerical structure behind the DCEE.
The ensemble-based DCEE  is rewritten only to clarify how the exploitation and exploration terms appear in the resulting optimization problem.
Then, the DCEE is not simply treated as a black-box nonlinear cost in the control input. 
It admits a convex-over-nonlinear form, which will be used in the next section to design the fast numerical solution method.


\subsection{Residual Representation}
\label{subsec:residual_representation}

Recall the reward function in DCEE
\begin{equation}
D_k(u) = \left\| y_{k+1|k}-\bar\gamma_{k+1|k} \right\|^2 + \frac1N \sum_{i=1}^N \left\| \tilde\gamma_{k+1|k}^i \right\|^2 .
\label{eq:Dk_strict_objective}
\end{equation}
The objective in \eqref{eq:Dk_strict_objective} can be written as the squared norm of a stacked residual vector. 
This form integrates the tracking term and the uncertainty term into a unified least-squares structure.

Define the exploitation residual as
\begin{equation}
r_{{\rm ex},k} := y_{k+1|k}-\bar\gamma_{k+1|k}, \label{eq:rex_dcee} 
\end{equation}
and define the uncertainty residual as
\begin{equation}
r_{{\rm unc},k} := \frac1{\sqrt N} \operatorname{col} \left( \tilde\gamma_{k+1|k}^1, \dots, \tilde\gamma_{k+1|k}^N
\right).
\label{eq:runc_dcee}
\end{equation}
The stacked DCEE residual map is then
\begin{equation}
F_k(u) := \operatorname{col} \left( r_{{\rm ex},k}, r_{{\rm unc},k} \right).
\label{eq:Rk_dcee}
\end{equation}
By construction,
\begin{equation}
D_k(u) = \|F_k(u)\|_2^2.
\label{eq:Dk_residual_norm}
\end{equation}

This representation is useful for two reasons. 
First, it shows that the exploitation and exploration components of DCEE share the same residual-based least-squares structure. 
Second, it isolates the nonlinear part of the problem inside the map \(F_k(u)\). 
The nonlinearity of \(F_k\) comes from the one-step system prediction, the candidate-input-induced ensemble update, and the optimal-operation map. 
The outer loss, however, is simply a squared norm.

For consistency with the standard least-squares and Gauss-Newton notation,
we introduce the scaled objective
\begin{equation}
L_{k}(u) := \frac12D_k(u) = \frac12\|F_k(u)\|_2^2.
\label{eq:f0k_def}
\end{equation}
The factor \(1/2\) does not affect the optimizer. 
It is introduced only to match the conventional notation used in least-squares and Gauss-Newton methods.


\subsection{Composite Convex-over-Nonlinear Form}
\label{subsec:composite_con_form}

At each sampling instant \(k\), the current state \(x_k\) and the current estimator ensemble \(\Theta_k\) are given by the measurements and the learning update. 
Hence, the reward function in DCEE is viewed as a parametric function of the candidate input \(u\), denoted by
\begin{equation}
L_k(u)=L(u;x_k,\Theta_k).
\end{equation}
In this sense, the convex-over-nonlinear structure below is stated with respect to the optimization variable \(u\), while \(x_k\) and \(\Theta_k\) are treated as fixed parameters of the current subproblem.

\begin{proposition}[Convex-over-nonlinear structure of reward function in DCEE]
\label{prop:con_structure_strict_dcee}
For fixed \(x_k\) and \(\Theta_k\), the reward function in DCEE is convex over nonlinear with respect to $u$.
\end{proposition}

\begin{proof}
Define
\begin{equation}
\phi(v):=\frac12\|v\|_2^2.
\label{eq:F0_phi0_dcee}
\end{equation}
The DCEE problem is therefore
\begin{equation}
\min_{u\in\mathbb R^{n_u}}
L_{k}(u) = \min_{u\in\mathbb R^{n_u}} \phi(F_{k}(u)).
\label{eq:unconstrained_dcee_composite_problem}
\end{equation} 
The result follows from the residual definitions \eqref{eq:rex_dcee}-\eqref{eq:F0_phi0_dcee}. 
The function \(\phi(v)=\frac12\|v\|_2^2\) is convex, and \(F_k(u)\) collects the nonlinear residual terms. Thus, \(L_k(u)\) has the stated composite form.
\end{proof}

\begin{remark}
Proposition~\ref{prop:con_structure_strict_dcee} does not claim global convexity of \(L_k(u)\) with respect to \(u\). 
The DCEE problem may still be nonconvex, since \(F_k\) is generally nonlinear. 
The role of the proposition is to identify a key structural property of the reward function in DCEE: the nonlinearity is contained in the inner residual map, whereas the outer loss is convex.
This convex-over-nonlinear structure provides useful numerical properties for optimization, as discussed in detail in the next section.
\end{remark}


\section{Structure-Exploiting Numerical Methods for DCEE}
\label{sec:structure_exploiting_methods_dcee}

The previous section showed that the reward function in DCEE can be written in a convex-over-nonlinear form. 
Inspired by \cite{messerer2021survey}, this section uses this structure to construct a direct numerical method for the DCEE problem.


\subsection{Sequential Convex Programming}
\label{subsec:sequential_convex_programming_dcee}

Sequential convex programming provides the structural principle for exploiting the composite form \eqref{eq:unconstrained_dcee_composite_problem}. 
Instead of approximating the entire nonlinear objective by an arbitrary quadratic model, SCP linearizes only the nonlinear inner map \(F_k\) and keeps the convex outer function \(\phi\) unchanged.

Let \(u^j\) be the $j$-th  inner iterate. 
Define the first-order approximation of the inner residual map as
\begin{equation}
F_{k}^{\rm lin}(u;u^j) = F_{k}(u^j) + J_{k}(u^j)(u-u^j),
\label{eq:F0_lin_dcee}
\end{equation}
with Jacobian
\begin{equation}
J_{k}(u) := \frac{\partial F_{k}(u)}{\partial u}.
\label{eq:J0k_dcee}
\end{equation}
The corresponding sequential convex approximation is
\begin{equation}
u^{j+1} \in \argmin_{u\in\mathbb R^{n_u}} \phi \left( F_{k}^{\rm lin}(u;u^j) \right).
\label{eq:dcee_scp_general}
\end{equation}
Since \(F_{k}^{\rm lin}(\cdot;u^j)\) is affine in \(u\) and \(\phi\) is convex, \eqref{eq:dcee_scp_general} is a convex optimization subproblem.

For the DCEE, the corresponding convex subproblem \eqref{eq:dcee_scp_general} becomes
\begin{equation}
u^{j+1} \in \argmin_{u\in\mathbb R^{n_u}} \frac12 \left\| F_{k}(u^j)+J_{k}(u^j)(u-u^j) \right\|_2^2 .
\label{eq:dcee_scp_ls}
\end{equation}
Equivalently, with $\Delta u^j:=u-u^j$, we obtain SCP subproblem (SCP-DCEE) of DCEE
\begin{equation}
\Delta u^j \in \argmin_{\Delta u\in\mathbb R^{n_u}} \frac12 \left\| F_{k}(u^j)+J_{k}(u^j)\Delta u \right\|_2^2 .
\label{eq:dcee_scp_delta}
\end{equation}

\begin{proposition}[Convexity of the SCP-DCEE subproblem]
\label{prop:convexity_dcee_scp}
For fixed \(x_k\), \(\Theta_k\), and inner iterate \(u^j\), the subproblem \eqref{eq:dcee_scp_general} is convex. 
For the DCEE, it reduces to the convex least-squares problem \eqref{eq:dcee_scp_delta}.
\end{proposition}

\begin{proof}
The linearized residual map \(F_{k}^{\rm lin}(\cdot;u^j)\) is affine.
The composition of a convex function with an affine map is convex. 
For \(\phi(v)=\frac12\|v\|_2^2\), the resulting objective is a convex quadratic function of \(\Delta u^j\).
\end{proof}

\begin{remark}
This proposition explains the first numerical advantage of the proposed formulation. 
Although the original DCEE problem is generally nonconvex, each subproblem is convex. 
\end{remark}


\subsection{Gauss-Newton Step and Generalized Hessian Approximation}
\label{subsec:gn_step_hessian_approximation}

The preceding subsection derived the SCP-DCEE subproblem from the convex-over-nonlinear viewpoint.
Since the outer loss of the reward function in DCEE is quadratic, this SCP subproblem is exactly a Gauss-Newton least-squares step. 
The generalized Gauss-Newton Hessian decomposition further explains the numerical advantage of this step.

Expanding the squared norm in \eqref{eq:dcee_scp_delta} gives
\begin{align}
& \frac12 \left\| F_{k}(u^j)+J_{k}(u^j)\Delta u \right\|_2^2 \nonumber\\
& = L_{k}(u^j) + \nabla L_{k}(u^j)^\top \Delta u + \frac12 \Delta u^\top J_{k}(u^j)^\top J_{k}(u^j) \Delta u,
\label{eq:scp_expansion_to_gn}
\end{align}
where
\begin{equation}
\nabla L_{k}(u^j) = J_{k}(u^j)^\top F_{k}(u^j).
\label{eq:gradient_ls_dcee}
\end{equation}
The constant term \(L_k(u^j)\) does not affect the minimizer. 
Hence the search direction is determined by the linear term \(\nabla L_k(u^j)^\top \Delta u\) and the quadratic term generated by the residual Jacobian. 
This gives the Gauss-Newton step
\begin{align}
\Delta u^j \in \argmin_{\Delta u\in\mathbb R^{n_u}} \quad & \frac12 \Delta u^\top B_{{\rm GN},k}(u^j) \Delta u + \nabla L_{k}(u^j)^\top \Delta u,
\label{eq:dcee_gn_quadratic_model}
\end{align}
with
\begin{equation}
B_{{\rm GN},k}(u^j) = J_{k}(u^j)^\top J_{k}(u^j).
\label{eq:GN_matrix_dcee}
\end{equation}
Thus, the curvature used in the local quadratic model is built only from the Jacobian of the DCEE residual map. 
This is different from a full Newton step, which would require second-order derivative information.

The Gauss-Newton matrix \(B_{{\rm GN}}\)  can also be understood from the generalized Gauss-Newton Hessian decomposition. 
For the composite objective $L_{k}(u)=\phi(F_{k}(u))$, the exact Hessian satisfies
\begin{equation}
\nabla^2 L_{k}(u) = B_{{\rm GGN},k}(u) + E_{{\rm GGN},k}(u),
\label{eq:ggn_hessian_split_dcee}
\end{equation}
where
\begin{equation}
B_{{\rm GGN},k}(u) = J_{k}(u)^\top \nabla^2\phi(F_{k}(u)) J_{k}(u)
\label{eq:BGGN_dcee}
\end{equation}
\begin{equation}
E_{{\rm GGN},k}(u) = \sum_{\ell=1}^{m} [\nabla\phi(F_k(u))]_\ell \nabla^2 F_{k,\ell}(u).
\label{eq:EGGN_dcee}
\end{equation}
Here \([\nabla\phi(F_k(u))]_\ell\) denotes the \(\ell\)-th component of the gradient of the outer function \(\phi\) evaluated at \(F_k(u)\), \(m\) is the dimension of \(F_{k}\), and \(F_{k,\ell}\) denotes its \(\ell\)-th component.
The first term \(B_{{\rm GGN},k}\) describes the curvature obtained by linearizing the inner residual map and keeping the curvature of the outer loss. 
The second term \(E_{{\rm GGN},k}\) contains the curvature of the nonlinear residual map itself. 

\begin{proposition}
\label{prop:gn_step_strict_dcee}
For the reward function in DCEE
\begin{equation}
L_k(u) = \frac12\|F_k(u)\|_2^2,
\end{equation}
the sequential convex approximation \eqref{eq:dcee_scp_delta} is equivalent
to the Gauss-Newton least-squares step \eqref{eq:dcee_gn_quadratic_model}.
Moreover, the generalized Gauss-Newton Hessian approximation is positive semidefinite.
\end{proposition}

\begin{proof}
The equivalence follows by expanding the squared norm in \eqref{eq:dcee_scp_delta}, which yields the quadratic model \eqref{eq:dcee_gn_quadratic_model}. 
Since the DCEE outer loss is the squared norm, we have
\begin{equation}
\nabla^2\phi(F_{k}(u)) = I.
\end{equation}
Hence, the generalized Gauss-Newton Hessian approximation becomes
\begin{equation}
B_{{\rm GGN},k}(u) = J_{k}(u)^\top J_{k}(u) = B_{{\rm GN},k}(u).
\label{eq:BGGN_equals_GN_dcee}
\end{equation}
Thus, the Gauss-Newton matrix used in \eqref{eq:dcee_gn_quadratic_model} is exactly the generalized Gauss-Newton curvature term of the composite reward function in DCEE.
The positive semidefiniteness follows immediately from $J_k(u)^\top J_k(u) \succeq 0$.
\end{proof}



\begin{remark}
This proposition shows the second numerical advantage of the DCEE formulation.
The Gauss-Newton step uses the positive-semidefinite approximation
\(J_k(u)^\top J_k(u)\), which is formed only from the first-order derivative of the DCEE residual map.
The second-order part omitted from the exact Hessian is \(E_{{\rm GGN},k}(u)\).
It contains second derivatives of the residual components weighted by \([\nabla\phi(F_k(u))]_\ell\), which equals \(F_{k,\ell}(u)\) for the squared-norm outer loss. 
Therefore, near a regular local minimizer with small residual, this omitted curvature term has a reduced influence.
Hence the indefinite second-order terms in the exact Hessian are not used, and no second derivatives are required.
This property is crucial for fast online solution of the numerical DCEE.
\end{remark}

Note that SCP, GN, and GGN discussed in this section play different roles in the numerical DCEE. 
SCP gives the basic approximation principle: linearize the nonlinear DCEE residual map and retain the convex outer loss. 
Since the DCEE outer loss is quadratic, the corresponding SCP subproblem reduces to a Gauss-Newton least-squares step. 
The generalized Gauss-Newton Hessian decomposition then explains why the curvature matrix in this step is positive semidefinite and depends only on first-order derivatives.


\subsection{Local Convergence Analysis}
\label{subsec:local_convergence_analysis_dcee}

The local convergence property of the Gauss-Newton iteration of DCEE is stated below. 
The statement concerns the  theoretical properties of structure-exploiting numerical DCEE. 
It is not a closed-loop stability result of the controlled system.
The stability and convergence of the numerical DCEE controller will be a focus of our future research.

\begin{theorem}[Local convergence of the numerical DCEE]
\label{thm:linear_local_convergence_dcee_gn}
Consider the smooth unconstrained DCEE subproblem
\begin{equation}
\min_{u\in\mathbb R^{n_u}} L_k(u) = \min_{u\in\mathbb R^{n_u}} \phi(F_k(u)).
\label{eq:dcee_smooth_con_problem_for_theorem}
\end{equation}
Assume that \(F_k\) and \(\phi\) are sufficiently smooth in a neighborhood of a local minimizer \(u_k^\ast\), and that
\begin{equation}
\nabla L_k(u_k^\ast)=0, \quad B_{{\rm GGN},k}(u_k^\ast)\succ0 .
\label{eq:dcee_local_minimizer_gn}
\end{equation}
Then \(u_k^\ast\) is a fixed point of the full-step DCEE-Gauss-Newton iteration, and the iteration is locally well-defined. 
Its local linear contraction-or divergence-rate is characterized by the smallest \(\alpha_k\ge0\) satisfying
\begin{equation}
-\alpha_k B_{{\rm GGN},k}(u_k^\ast) \preceq E_{{\rm GGN},k}(u_k^\ast) \preceq \alpha_k B_{{\rm GGN},k}(u_k^\ast).
\label{eq:dcee_gn_lmi_rate}
\end{equation}
If \(\alpha_k<1\), the iteration is locally Q-linearly convergent to \(u_k^\ast\). 
A sufficient LMI condition is
\begin{equation}
\frac{1}{1+\alpha_k} \nabla^2 L_k(u_k^\ast) \preceq B_{{\rm GGN},k}(u_k^\ast) \preceq \frac{1}{1-\alpha_k} \nabla^2 L_k(u_k^\ast).
\label{eq:dcee_gn_sufficient_lmi}
\end{equation}
\end{theorem}

\begin{proof}
We have shown that the DCEE problem has the smooth unconstrained convex-over-nonlinear form studied in the local convergence theory of SCP/GGN methods. 
Moreover, in the DCEE case, \(\phi(v)=\frac12\|v\|_2^2\), so the SCP-DCEE step coincides with the Gauss-Newton least-squares step, and the Hessian split is
\begin{equation}
\nabla^2 L_k(u)=B_{{\rm GGN},k}(u)+E_{{\rm GGN},k}(u). \notag
\end{equation}
Therefore, under \eqref{eq:dcee_local_minimizer_gn}, the standard local convergence theorem for unconstrained GGN/SCP applies directly and yields \eqref{eq:dcee_gn_lmi_rate}.
The sufficient condition \eqref{eq:dcee_gn_sufficient_lmi} follows by substituting \(E_{{\rm GGN},k}=\nabla^2 L_k-B_{{\rm GGN},k}\). 
Detailed derivations are given in the local convergence analysis of generalized Gauss-Newton and sequential convex programming \cite{messerer2021survey,diehl2019local}.
\end{proof}


\begin{algorithm}[t]
\small
\caption{Numerical DCEE}
\label{alg:direct_numerical_dcee_gn}
\begin{algorithmic}[1]
\REQUIRE Current state measurement \(x_k\), previous input \(u_{k-1}\), current estimator ensemble \(\Theta_k\), maximum number of inner iterations \(N_{\rm it}\), tolerance \(\epsilon_{\rm GN}>0\).
\ENSURE Applied control input \(u_k\), updated estimator ensemble \(\Theta_{k+1}\).

\STATE Initialize \(u_k^0 \leftarrow u_{k-1}\),
\(j\leftarrow 0\), \(\epsilon_k \leftarrow +\infty\).

\WHILE{\(j<N_{\rm it}\) and \(\epsilon_k>\epsilon_{\rm GN}\)}
    \STATE Predict \(y_{k+1|k}(u_k^j)\).
    \STATE Compute the candidate-input-induced predicted estimator ensemble $\left\{\hat\theta_{k+1|k}^i(u_k^j)\right\}_{i=1}^N$.
    \STATE Compute the predicted optimal operating ensemble $\left\{\hat\gamma_{k+1|k}^i(u_k^j)\right\}_{i=1}^N$.
    \STATE Construct the DCEE residual \(F_k(u_k^j)\) and its Jacobian
    \(J_k(u_k^j)\).
    \STATE Solve the Gauss-Newton least-squares step
    \[
    \Delta u_k^j \in
    \argmin_{\Delta u}
    \frac12
    \left\|
    F_k(u_k^j)+J_k(u_k^j)\Delta u
    \right\|_2^2 .
    \]
    \STATE Compute the stopping measure
    \[
    \epsilon_k \leftarrow
    \frac{\|\Delta u_k^j\|}{1+\|u_k^j\|}.
    \]
    \STATE Update $u_k^{j+1} \leftarrow u_k^j+\Delta u_k^j$.
    \STATE Set \(j\leftarrow j+1\).
\ENDWHILE

\STATE \textbf{Apply} \(u_k\leftarrow u_k^j\).
\STATE Observe the next measurement and reward pair \((y_{k+1},\mathcal R_{k+1})\).
\STATE Update the estimator ensemble from \(\Theta_k\) to \(\Theta_{k+1}\) using the measured pair \((y_{k+1},\mathcal R_{k+1})\).

\STATE \textbf{Return} \(u_k,\Theta_{k+1}\).
\end{algorithmic}
\end{algorithm}

\subsection{Numerical DCEE Implementation}
\label{subsec:numerical_method_implementation}

The preceding subsections lead to a fast and accurate numerical implementation of the DCEE problem. 
Algorithm~\ref{alg:direct_numerical_dcee_gn} summarizes the online procedure.
The control input is obtained after a few Gauss-Newton iterations.
In a real-time feedback setting, solving the numerical subproblem to full convergence is not required. 
The iteration stops when the relative step size is smaller than the prescribed tolerance, or when the maximum number of iterations is reached.
Noted that the Jacobian can be obtained by automatic differentiation. 
The proposed method only requires first-order derivatives of the DCEE residual map. 
There is no need to compute second derivatives for the prediction, the ensemble-learning update, or the optimal-operation map.

\section{Case Study and Results}
\label{sec:case_study_results}

This section validates the proposed numerical DCEE method on a vehicle cruising auto-optimization problem. 
First, the vehicle cruising auto-optimization problem formulation is introduced. 
Then, the proposed numerical DCEE method is compared with existing methods, with emphasis on closed-loop performance and computational efficiency. 
Finally, hardware-in-the-loop experiments are conducted to verify the real-time feasibility of the proposed method under varying driving conditions.


\subsection{Vehicle Cruising Auto-Optimization under Unknown Environment}
\label{subsec:eco_cruising}

We consider an autonomous electric vehicle operating under unknown road-load conditions, such as tailwind or headwind, variations in aerodynamic drag, payload changes, and rolling-resistance variations. 
These factors affect the vehicle speed-force operating point and therefore influence the energy-consumption characteristics. 
This setting is consistent with eco-driving and predictive cruise control, where the vehicle speed is optimized by jointly considering longitudinal dynamics, road environment, traffic requirements, and energy-consumption models \cite{sciarretta2020energy,han2019fundamentals,zhang2023energy}.

The longitudinal speed dynamics are modeled as
\begin{equation}
v_{k+1}= f(v_k, F_{m,k}, \xi),
\label{eq:eco_vehicle_dynamics}
\end{equation}
where \(v_k\) is the vehicle speed, \(F_{m,k}\) is the motor traction force, and \(\xi\) denotes the unknown lumped road-load effect. 

To motivate the reward model, let $z=\frac{v}{v_{\rm scale}}$, where \(v_{\rm scale}\) is a constant normalization speed.. 
A local second-order approximation of the electric energy-consumption map \cite{han2019fundamentals} can be written as
\begin{align}
\mathcal E(z,F_{m};\xi) 
=
p_0(\xi)
+p_1(\xi)z^2
+p_2(\xi)F_{m}\nonumber \\
+p_3(\xi)F_{m}^2
+p_4(\xi)zF_{m} .
\label{eq:eco_energy_map}
\end{align}
The simplified fitted map \(p_1z^2+p_2F_{m}+p_3F_{m}^2\) is a special case of \eqref{eq:eco_energy_map} \cite{han2019fundamentals}.

Eco-cruising is not modeled as a pure energy-minimization problem, since such a criterion would prefer low vehicle speeds. 
Therefore, an eco-driving cost is defined to balance energy efficiency and mobility benefit
\begin{equation}
\mathcal R(z,F_{m};\xi) = q_z z-\lambda_E\mathcal E(z,F_{m};\xi),
\label{eq:eco_net_reward}
\end{equation}
where \(q_z z\) represents the mobility benefit and \(\lambda_E>0\) weights the energy efficiency.

Following the longitudinal force-balance formulation commonly used in eco-driving studies~\cite{han2019fundamentals}, the steady cruising motor force is determined by the local road-load condition. 
Around a given cruising region, this steady-state force relation is locally linearized as
\begin{equation}
F_m^{\rm ss}(z;\xi)\approx a_1(\xi)z+a_0(\xi).
\label{eq:eco_steady_force_norm}
\end{equation}
Substituting \eqref{eq:eco_steady_force_norm} into \eqref{eq:eco_net_reward} yields a reduced speed-dependent reward
\begin{equation}
\mathcal R_{\rm ss}(z;\xi) = \mathcal R\big(z,F_{m}^{\rm ss}(z;\xi);\xi\big) = \theta_1^\ast z^2+\theta_2^\ast z+\theta_3^\ast .
\label{eq:eco_reward_quadratic_from_map}
\end{equation}
Thus, the quadratic reward used in this experiment is a local reduced approximation of a speed-force energy map under steady cruising, rather than an artificially prescribed optimal-speed objective.
The reduced steady-cruising reward can be represented by 
\begin{equation}
\mathcal R_{\rm ss}(z;\xi) = C_R - w_z(z-z^\ast(\xi))^2, \quad w_z(\xi)>0 ,
\label{eq:eco_reward_peak_form}
\end{equation}
where \(C_R\) is a fixed reward offset. 
This offset has no influence on the optimal cruising speed and is introduced only for reward normalization.
Expanding \eqref{eq:eco_reward_peak_form} gives the linear-in-parameters form
\begin{equation}
\mathcal R(\theta,v) = \psi(v)^\top\theta,
\quad
\psi(v)=
\begin{bmatrix}
z^2\\ z\\ 1
\end{bmatrix}.
\label{eq:eco_reward_linear}
\end{equation}
Thus, the unknown environment is represented by $\theta^\ast$. The associated optimal speed is 
\begin{equation}
\gamma^\ast = \Gamma(\theta^\ast) = v_{\rm scale} \left( -\frac{\theta_2^\ast}{2\theta_1^\ast} \right).
\label{eq:eco_gamma_map}
\end{equation}
Therefore, variations in headwind, tailwind, aerodynamic drag, payload, or rolling resistance are represented through changes in \(\theta^\ast\), and changes in \(\theta^\ast\) further shift the unknown local optimal cruising speed.

\subsection{Performance Comparison}
\label{subsec:performance_comparison}

\subsubsection{Closed-loop auto-optimization performance}

\begin{figure}[t]
    \centering
    \includegraphics[width=\linewidth]{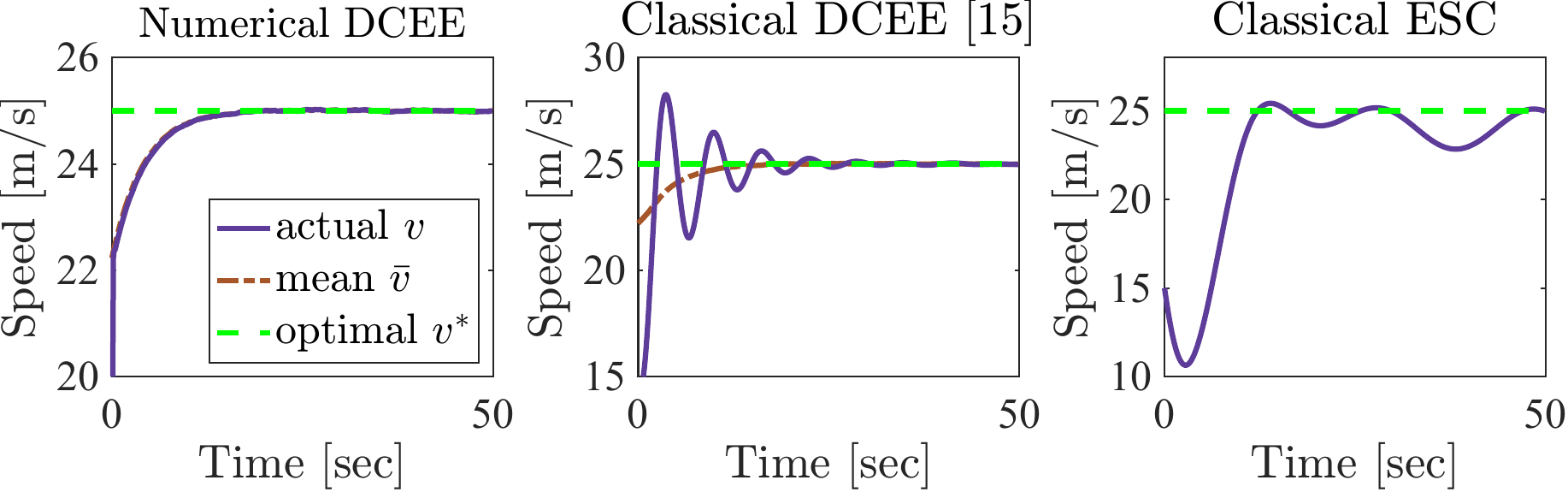}
    \caption{Closed-loop eco-cruising performance comparison among the proposed Numerical DCEE, Classical DCEE \cite{li2025dual}, and Classical ESC methods.}
    \label{fig:eco_speed_Numerical_DCEE}
\end{figure}

\begin{figure}[t]
    \centering
    \includegraphics[width=0.85\linewidth]{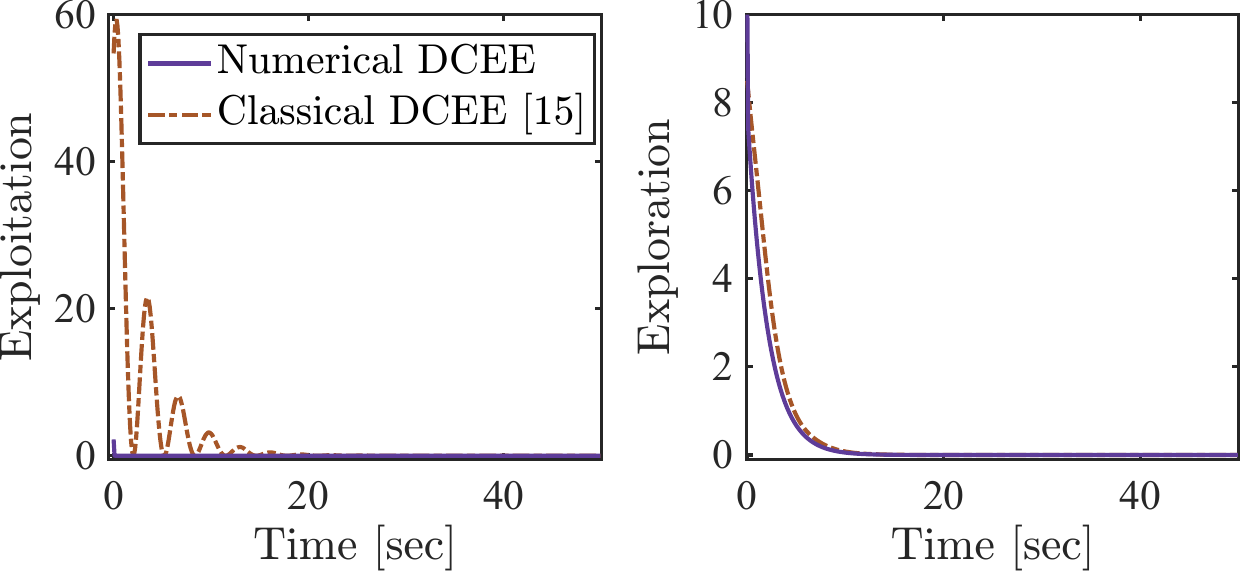}
    \caption{Comparison of the exploitation and exploration terms for numerical DCEE and classical DCEE \cite{li2025dual}.}
    \label{fig:eco_dual_terms}
\end{figure}

\begin{figure}[t]
    \centering
    \includegraphics[width=0.6\linewidth]{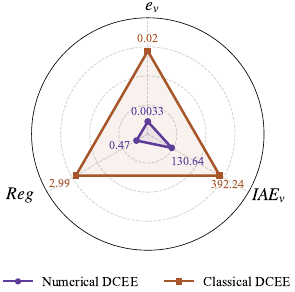}
    \caption{Radar-chart comparison of Numerical DCEE and Classical DCEE using three quantitative metrics: terminal speed error \(e_v\), accumulated absolute speed error \({IAE}_v\), and cumulative regret \({Reg}\). Lower values indicate better performance.}
    \label{fig:eco_radar_metrics}
\end{figure}

\begin{figure}[t]
    \centering
    \includegraphics[width=0.8\linewidth]{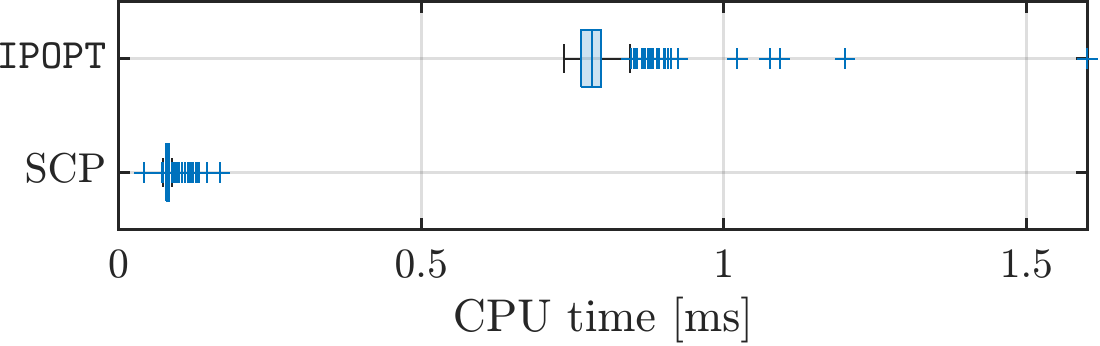}
    \caption{CPU time comparison between the proposed SCP-DCEE method and direct \texttt{IPOPT} solution of the original nonconvex DCEE.}
    \label{fig:cpu_time_gn_ipopt}
\end{figure}

We compare the proposed numerical DCEE method with two baseline methods: a classical DCEE implementation based on an explicit gradient-type update \cite{li2025dual}, and the classical extremum seeking control (ESC) method.
This comparison is intended to evaluate both the closed-loop auto-optimization performance and the effectiveness of the proposed structure-exploiting numerical solution method.

Fig. \ref{fig:eco_speed_Numerical_DCEE} shows the closed-loop speed responses of the three methods. 
The proposed numerical DCEE method drives the vehicle speed smoothly toward the unknown optimal cruising speed \(v^\ast\). 
The estimated mean optimal speed \(\bar v\) also consistently converges to the true optimum, indicating that the ensemble learning mechanism successfully identifies the unknown operating condition during closed-loop operation. 
In contrast, the classical DCEE method reaches the same optimum only after transient oscillations. 
Its explicit gradient-type update is more sensitive to the local scaling of the DCEE objective and produces a larger overshoot before stabilization. 
The classical ESC baseline also approaches the optimal speed on average, but it relies on persistent perturbations for gradient estimation. 
As a result, its response exhibits clear oscillations around the optimum, leading to larger tracking errors and cumulative performance loss.
These results show that the proposed direct numerical DCEE method improves the transient performance of the closed-loop system while preserving the exploration-exploitation behavior of DCEE.

Figure \ref{fig:eco_dual_terms} further compares the exploitation and exploration terms of the proposed numerical DCEE method and the classical DCEE method. 
The exploitation term measures the deviation between the predicted speed and the predicted mean optimal speed, while the exploration term measures the uncertainty of the predicted optimal operating condition. 
The proposed method keeps the exploitation term consistently small during the transient phase, which explains the smooth speed response. 
In contrast, the classical DCEE method produces a large peak in the exploitation term at the initial stage, corresponding to the aggressive transient correction and overshoot in the speed trajectory. 
As the ensemble learner gradually learns the unknown reward parameters, the exploration terms of both methods decrease. 
However, the proposed method reduces uncertainty without introducing a large exploitation cost. 
This validates the numerical advantage of directly optimizing the DCEE objective through the structure-exploiting SCP method.

To quantify the performance, two control-related metrics and one auto-optimization metric are used. 
The terminal speed error is defined as $e_v = \left|v_f-v_f^\ast\right|$, 
and the accumulated absolute speed error is ${IAE}_v = \sum \left|v_k-v_k^\ast\right|$. 
The auto-optimization performance is evaluated by the cumulative regret ${Reg}=  \sum \left[ \mathcal R(\theta_k^\ast,v_k^\ast) - \mathcal R(\theta_k^\ast,v_k) \right]$, 
which measures the cumulative reward loss caused by operating away from the true optimal speed. Smaller values of all three metrics indicate better performance.
The radar plot in Fig.~\ref{fig:eco_radar_metrics} presents a quantitative comparison between numerical DCEE and classical DCEE. 
The proposed method achieves \(e_v=0.0033\), \({IAE}_v=130.64\), and \({Reg}=0.47\), whereas classical DCEE gives \(e_v=0.02\), \({IAE}_v=392.24\), and \({Reg}=2.99\).
Therefore, compared with classical DCEE, the proposed numerical DCEE reduces the terminal speed error by approximately \(83.5\%\), the cumulative speed error by approximately \(66.7\%\), and the cumulative regret by approximately \(84.3\%\). 
These quantitative results are consistent with the trajectory-level observations: compared with the explicit gradient implementation, directly solving the residual-based DCEE subproblem leads to a smoother transient process, smaller tracking error, and lower reward loss.

\subsubsection{Computational Efficiency}

We further compare the computational efficiency of the proposed SCP-based DCEE method with an implementation that directly solves the original nonlinear DCEE using \texttt{IPOPT} \cite{wachter2006implementation} (\texttt{IPOPT}-DCEE). 
In SCP-DCEE, the residual Jacobian matrix is obtained by \texttt{CasADi} automatic differentiation \cite{andersson2018casadi}. 
The control input is then computed using the numerical DCEE procedure described in Algorithm~\ref{alg:direct_numerical_dcee_gn}.
Fig.~\ref{fig:cpu_time_gn_ipopt} shows the CPU-time distributions of SCP-DCEE and \texttt{IPOPT}-DCEE. 
Compared with \texttt{IPOPT}-DCEE, SCP-DCEE is approximately \(9.63\times\) faster on average, reducing the average CPU time from \(0.79\,{\rm ms}\) to \(0.082\,{\rm ms}\). 
It is also approximately \(9.41\times\) faster in the worst case, reducing the maximum CPU time from \(1.60\,{\rm ms}\) to \(0.17\,{\rm ms}\).
More importantly, the computation-time distribution of SCP-DCEE is more concentrated, whereas that of \texttt{IPOPT}-DCEE is more spread out and contains several high-time outliers.
This indicates that exploiting the convex-over-nonlinear structure of the DCEE can lead to more predictable online computational performance.
It should be noted that the classical DCEE implementation based on an explicit one-step gradient update is computationally cheaper, since it does not require solving an online optimization problem. 
However, as shown in the previous comparisons, this computational simplicity comes at the cost of larger transient oscillations and higher cumulative regret. 
The proposed SCP-DCEE method provides a fast and accurate numerical implementation. 


\subsection{Hardware-in-the-loop Experiment}

To verify the real-time capability of the proposed numerical method on an embedded controller, HiL experiments are conducted on the platform shown in Fig.~\ref{fig:HiL_Platform}. 
The experimental platform consists of two parts: the plant and the controller. 
The plant runs on a \texttt{SpeedGoat} real-time target machine. 
The controller is deployed on a prototype embedded vehicle controller based on the \texttt{dSPACE/MicroAutoBox} DS-1403 platform, which is equipped with a Texas Instruments AM5K2E04 processor. 
Communication between the plant and the controller is implemented through a CAN bus.

Figure~\ref{fig:HiL_result} shows the HiL results under changing driving conditions.
The optimal speed changes from \(25\,\mathrm{m/s}\) to \(20\,\mathrm{m/s}\), and then to \(30\,\mathrm{m/s}\).
The proposed controller tracks these changes smoothly and drives the actual speed close to the updated optimal speed.
The estimated mean optimal speed follows the shifted optimum, and the actual vehicle speed tracks it smoothly.
The control input changes according to the updated operating condition and remains smooth during the whole experiment.
The computation time remains below \(83~\mu\mathrm{s}\), indicating strong potential for real-world deployment of the proposed numerical DCEE.
It should be noted that this computation time is shorter than the simulation CPU time reported in Section~\ref{subsec:performance_comparison}, because the implementation is compiled into efficient C code and uses \texttt{BLASFEO}~\cite{frison2018blasfeo}, the basic linear algebra subroutines for embedded optimization library.
These implementation choices enable efficient execution on the vehicle-embedded CPU.

\begin{figure}[t]
    \centering
    \includegraphics[width=0.9\linewidth]{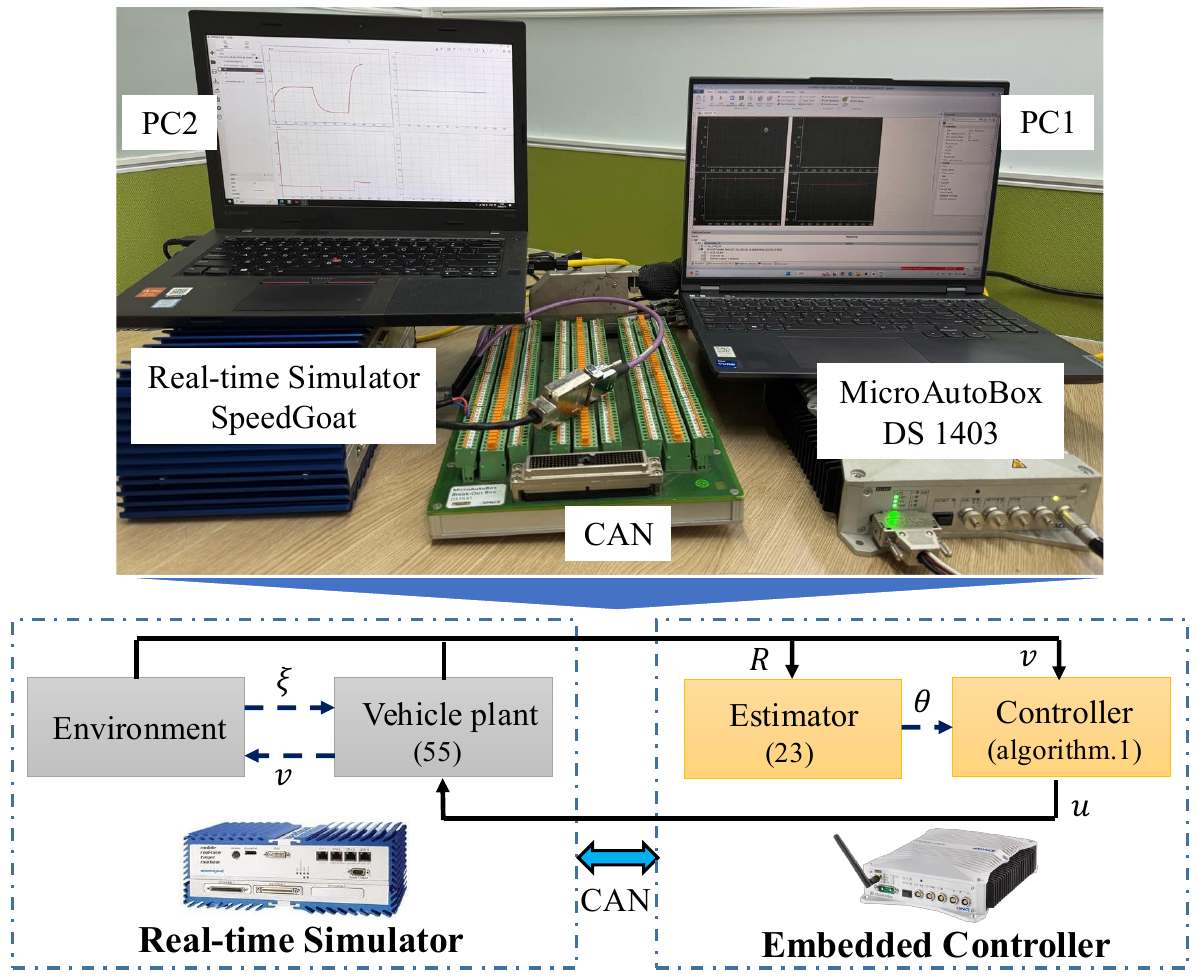}
    \caption{The HiL experiment platform.}
    \label{fig:HiL_Platform}
\end{figure}

\begin{figure}[t]
    \centering
    \includegraphics[width=\linewidth]{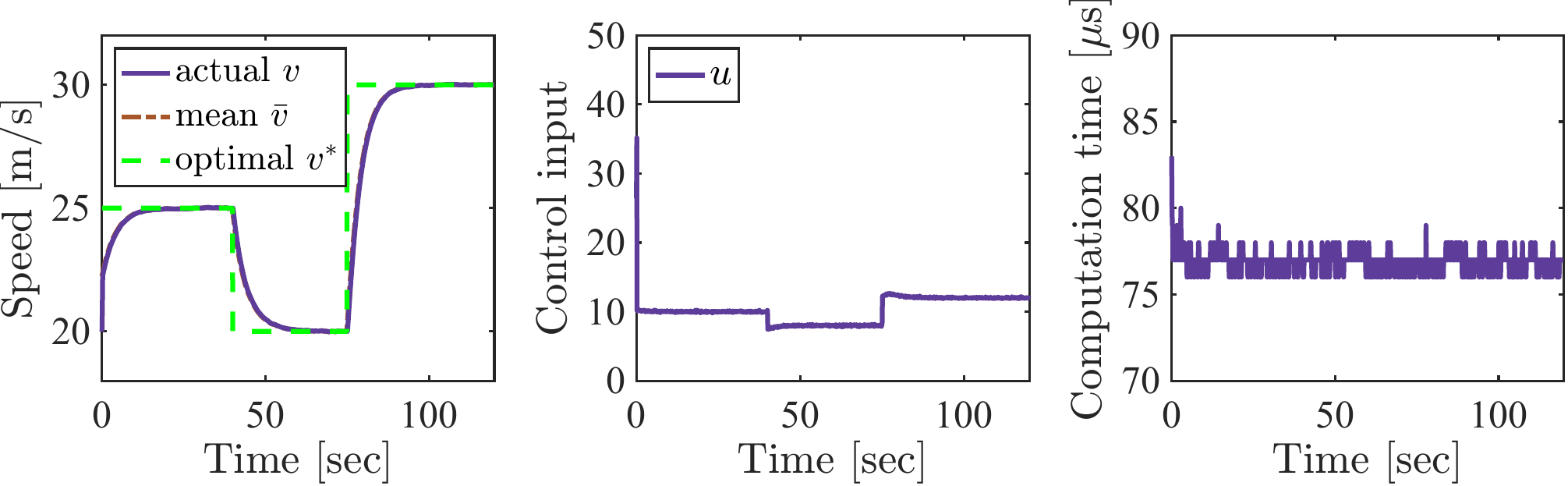}
    \caption{The experiment result of HiL under the changing driving condition. The max/average computational times are \(83~\mu\mathrm{s}\) and \(77~\mu\mathrm{s}\).}
    \label{fig:HiL_result}
\end{figure}


\section{Conclusion}
\label{sec:conclusion}

This paper proposes a structure-exploiting numerical dual control for exploration and exploitation (DCEE) method for auto-optimization control in unknown environments. 
The main finding is that the DCEE  has an inherent convex-over-nonlinear structure, in which the exploitation and exploration terms jointly form a unified nonlinear residual mapping equipped with a convex quadratic outer loss. 
By exploiting this structure, the DCEE problem is directly solved in the physical input space through a sequential convex programming procedure, where each inner iteration reduces to a Gauss-Newton least-squares step.
The resulting generalized Gauss-Newton Hessian approximation is positive semidefinite and requires only first-order derivatives. 
In the vehicle eco-cruising case study, compared with classical DCEE, the proposed method reduces the terminal speed error, cumulative speed error, and cumulative regret by \(83.5\%\), \(66.7\%\), and \(84.3\%\), respectively. 
Compared with direct numerical solution using the mature nonlinear programming solver \texttt{IPOPT}, the proposed numerical method improves the average computation speed by approximately \(9.63\) times and the worst-case computation speed by approximately \(9.41\) times. 
The hardware-in-the-loop results further verify the real-time feasibility of the proposed method under varying driving conditions, achieving a microsecond-level maximum computation time of only \(83~\mu\mathrm{s}\) on a vehicle embedded CPU.
Future work will focus on constrained DCEE to handle input and state constraints. 
Another important direction is to develop closed-loop stability and convergence analysis for numerical DCEE framework.


\bibliographystyle{IEEEtran}

\bibliography{global.bib}   

\begin{IEEEbiography}
[{\includegraphics[width=1in,height=1.25in,clip,keepaspectratio]{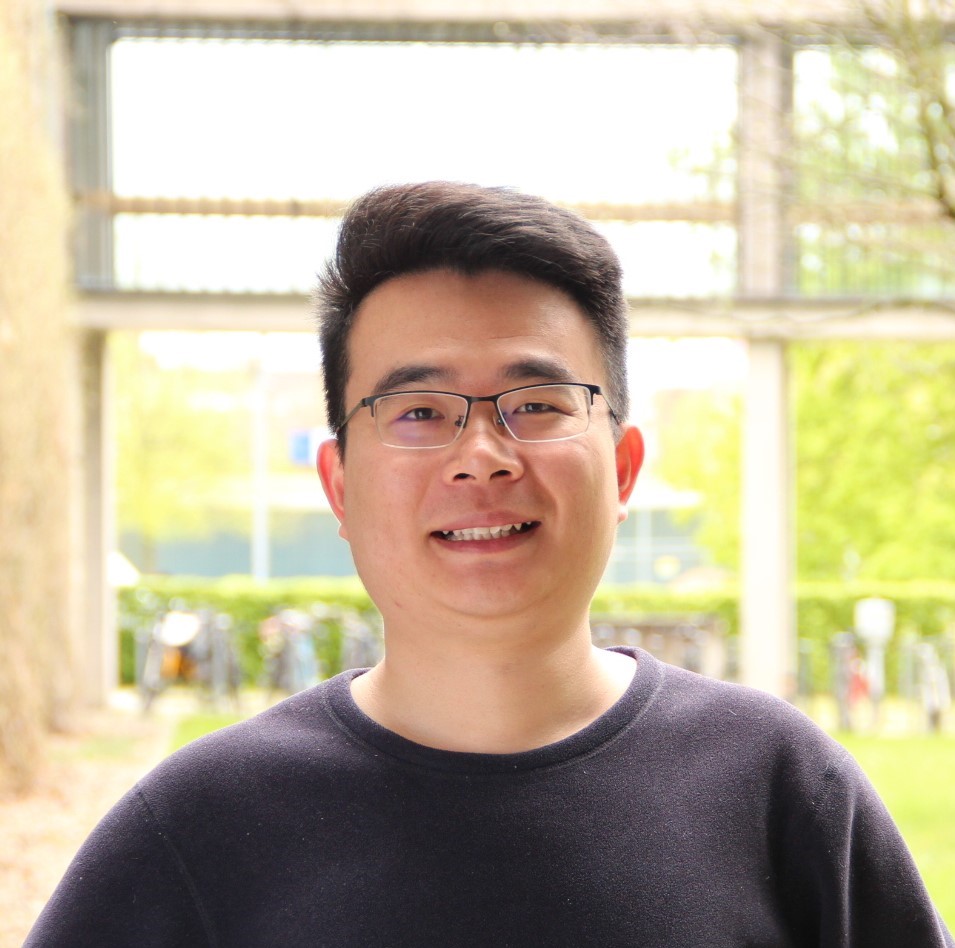}}]{Shiying Dong} received the B.S. degree in automation and the Ph.D. degree in control science and engineering from Jilin University, Changchun, China, in 2017 and 2024, respectively. He has been a joint Ph.D. student with the Systems Control and Optimization Laboratory, University of Freiburg, Germany, from 2022 to 2023. He was a Research Fellow with the Department of Electrical and Computer Engineering, National University of Singapore. He is currently a Postdoctoral Fellow with the Department of Aeronautical and Aviation Engineering, The Hong Kong Polytechnic University. His main research interest includes numerical optimization and learning-based predictive control, and their applications in autonomous vehicles.
\end{IEEEbiography}
\vspace{0.2mm}

\begin{IEEEbiography}
    [{\includegraphics[width=1in,height=1.25in,keepaspectratio]{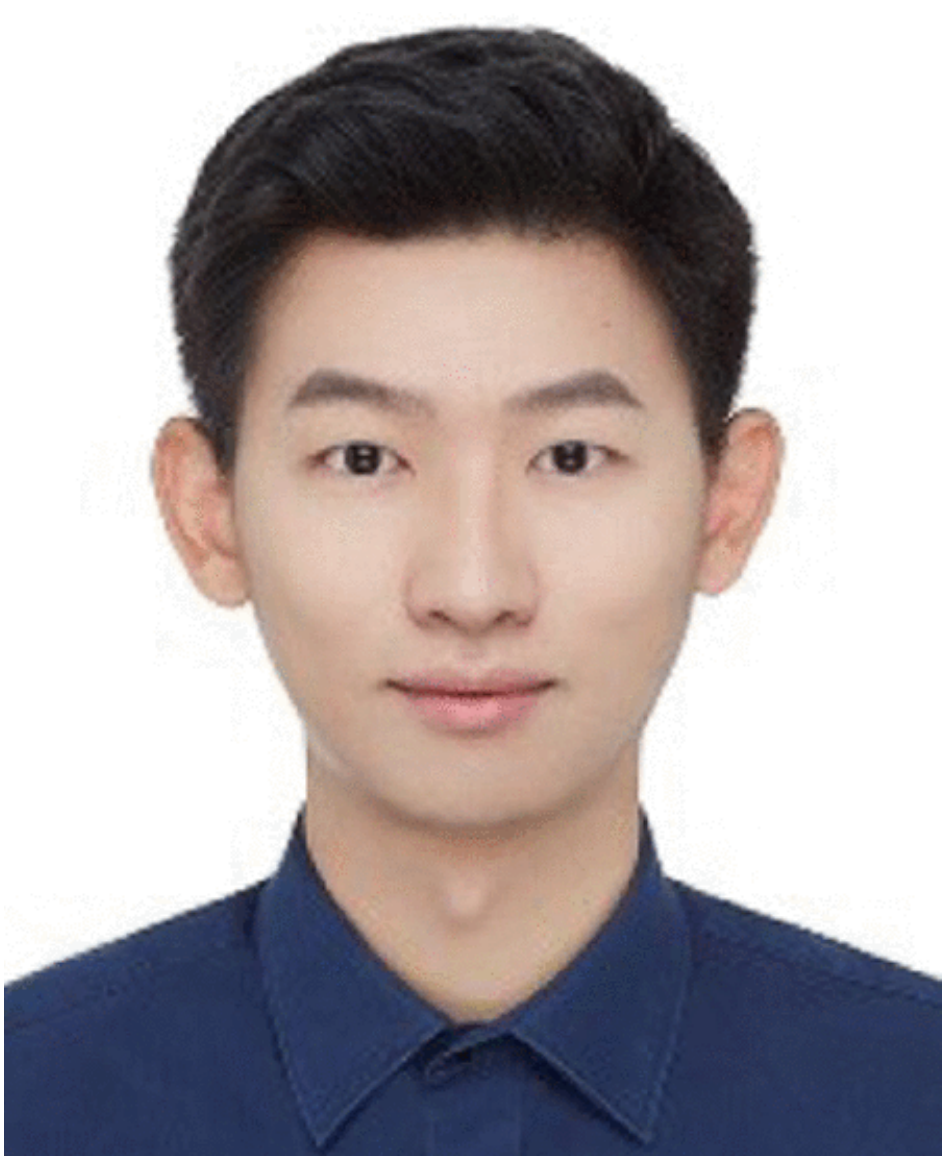}}]{Haoyang Yang} (IEEE Member) received the B.Eng. degree in measurement and control technique and instruments from Harbin Institute of Technology, Harbin, China, in 2017, and the Ph.D. degree in control science and engineering with the specialization in guidance and control from Beihang University, Beijing, China in 2023. He is currently a Research Assistant Professor in the Department of Aeronautical and Aviation Engineering at The Hong Kong Polytechnic University. Prior to this, he was a Postdoctoral Fellow in the same department and a Research Fellow at the University of Warwick, U.K. His research interests include safety-critical control, intelligent control, and their applications in unmanned systems.
\end{IEEEbiography}
\vspace{0.2mm}

\begin{IEEEbiography}
[{\includegraphics[width=1in,height=1.25in,clip,keepaspectratio]{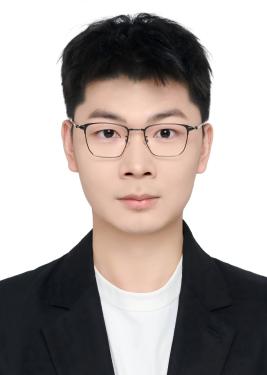}}]{Qiwei Liu} received the B.Sc. degree in automation, the Ph.D. degree in control science and engineering from East China University of Science and Technology, Shanghai, China, in 2021 and 2025, respectively. He is currently a Postdoctoral Fellow with the Department of Aeronautical and Aviation Engineering, The Hong Kong Polytechnic University. His research interests include neural networks, multi-agent systems and reinforcement learning.
\end{IEEEbiography}
\vspace{0.2mm}

\begin{IEEEbiography}
[{\includegraphics[width=1in,height=1.25in,clip,keepaspectratio]{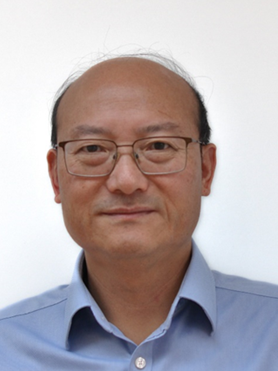}}]{Wen-Hua Chen}
(Fellow, IEEE) is currently the Chair Professor of Robotics and Autonomous Systems, the Director of the Research Centre for Low Altitude Economy, The Hong Kong Polytechnic University, Hong Kong, and the Chair of Autonomous Vehicles with the Department of Aeronautical and Automotive Engineering, Loughborough University, Loughborough, U.K. He is also a Chartered Engineer and a fellow of the Institution of Mechanical Engineers and the Institution of Engineering and Technology, U.K. He holds a five-year EPSRC Established Career Fellowship. He has authored or co-authored 350 articles and two books. He has considerable experience in control, signal processing, and artificial intelligence, and their applications in robotics, aerospace, and automotive systems. He has contributed significantly to a number of areas, including disturbance-observer-based control, model predictive control, and autonomous aerial vehicles.
\end{IEEEbiography}

\end{document}